%% file: main.tex
\documentclass{article}

\usepackage[final]{neurips_2019}

\usepackage[utf8]{inputenc} 
\usepackage[T1]{fontenc}    
\usepackage{hyperref}       
\usepackage{url}            
\usepackage{booktabs}       
\usepackage{amsfonts}       
\usepackage{nicefrac}       
\usepackage{microtype}      
\usepackage{graphicx}
\usepackage{amsmath}
\usepackage{amsthm}
\usepackage{caption}
\usepackage{xcolor}
\usepackage{thmtools}
\usepackage{thm-restate}
\usepackage[ruled,vlined]{algorithm2e}
\SetKwInput{KwInput}{Input}
\SetKwInOut{KwParameter}{Parameter}
\SetKwInput{KwOutput}{Output}

\SetCommentSty{mycommfont}

\newtheorem{lemma}{Lemma}

\title{Fully Parameterized Quantile Function for Distributional Reinforcement Learning}

\author{
 Derek Yang\thanks{Contributed during internship at Microsoft Research.} \\
 UC San Diego\\
\texttt{dyang1206@gmail.com} \\
\And
Li Zhao \\
Microsoft Research \\
\texttt{lizo@microsoft.com} \\
\And
Zichuan Lin \\
Tsinghua University \\
\texttt{linzc16@mails.tsinghua.edu.cn} \\
\And
Tao Qin \\
Microsoft Research \\
\texttt{taoqin@microsoft.com} \\
\And
Jiang Bian \\
Microsoft Research \\
\texttt{jiang.bian@microsoft.com} \\
\And
Tieyan Liu \\
Microsoft Research \\
\texttt{tyliu@microsoft.com} \\
}

\begin{document}

\maketitle

\begin{abstract}
Distributional Reinforcement Learning (RL) differs from traditional RL in that, rather than the expectation of total returns, it estimates distributions and has achieved state-of-the-art performance on Atari Games. The key challenge in practical distributional RL algorithms lies in how to parameterize estimated distributions so as to better approximate the true continuous distribution. Existing distributional RL algorithms parameterize either the probability side or the return value side of the distribution function, leaving the other side uniformly fixed as in C51, QR-DQN or randomly sampled as in IQN. In this paper, we propose fully parameterized quantile function that parameterizes both the quantile fraction axis (i.e., the x-axis) and the value axis (i.e., y-axis) for distributional RL. Our algorithm contains a fraction proposal network that generates a discrete set of quantile fractions and a quantile value network that gives corresponding quantile values. The two networks are jointly trained to find the best approximation of the true distribution. Experiments on 55 Atari Games show that our algorithm significantly outperforms existing distributional RL algorithms and creates a new record for the Atari Learning Environment for non-distributed agents. 

\end{abstract}

\section{Introduction}
Distributional reinforcement learning~\citep{jaquette1973markov,sobel1982variance, white1988mean, morimura2010nonparametric, bellemare2017distributional} differs from value-based reinforcement learning in that, instead of focusing only on the expectation of the return, distributional reinforcement learning also takes the intrinsic randomness of returns within the framework into consideration~\citep{bellemare2017distributional,dabney2018distributional, dabney2018implicit, rowland2018analysis}. The randomness comes from both the environment itself and agent's policy. Distributional RL algorithms characterize the total return as random variable and estimate the distribution of such random variable, while traditional Q-learning algorithms estimate only the mean (i.e., traditional value function) of such random variable.

The main challenge of distributional RL algorithm is how to parameterize and approximate the distribution. In Categorical DQN~\citep{bellemare2017distributional}(C51), the possible returns are limited to a discrete set of fixed values, and the probability of each value is learned through interacting with environments. C51 out-performs all previous variants of DQN on a set of 57 Atari 2600 games in the Arcade Learning Environment (ALE)~\citep{bellemare2013arcade}.
Another approach for distributional reinforcement learning is to estimate the quantile values instead. 
\cite{dabney2018distributional} proposed QR-DQN to compute the return quantiles on fixed, uniform quantile fractions using quantile regression and minimize the quantile Huber loss~\citep{huber:1964} between the Bellman updated distribution and current return distribution. Unlike C51, QR-DQN has no restrictions or bound for value and achieves significant improvements over C51. However, both C51 and QR-DQN approximate the distribution function or quantile function on fixed locations, either value or probability. \cite{dabney2018implicit} propose learning the quantile values for sampled quantile fractions rather than fixed ones with an implicit quantile value network (IQN) that maps from quantile fractions to quantile values. With sufficient network capacity and infinite number of quantiles, IQN is able to approximate the full quantile function.

However, it is impossible to have infinite quantiles in practice. With limited number of quantile fractions, efficiency and effectiveness of the samples must be reconsidered. The sampling method in IQN mainly helps training the implicit quantile value network rather than approximating the full quantile function, and thus there is no guarantee in that sampled probabilities would provide better quantile function approximation than fixed probabilities. 

In this work, we extend the method in \cite{dabney2018distributional} and \cite{dabney2018implicit} and propose to fully parameterize the quantile function. By fully parameterization, we mean that unlike QR-DQN and IQN where quantile fractions are fixed or sampled and only the corresponding quantile values are parameterized, both quantile fractions and corresponding quantile values in our algorithm are parameterized. In addition to a quantile value network similar to IQN that maps quantile fractions to corresponding quantile values, we propose a fraction proposal network that generates quantile fractions for each state-action pair. The fraction proposal network is trained so that as the true distribution is approximated, the $1$-Wasserstein distance between the approximated distribution and the true distribution is minimized. Given the proposed fractions generated by the fraction proposal network, we can learn the quantile value network by quantile regression. With self-adjusting fractions, we can approximate the true distribution better than with fixed or sampled fractions.


We begin with related works and backgrounds of distributional RL in Section 2. We describe our algorithm in Section 3 and provide experiment results of our algorithm on the ALE environment~\citep{bellemare2013arcade} in Section 4. At last, we discuss the future extension of our work, and conclude our work in Section 5.

\section{Background and Related Work}
We consider the standard reinforcement learning setting where agent-environment interactions are modeled as a Markov Decision Process $(\mathcal{X}, \mathcal{A}, R, P, \gamma)$~\citep{Puterman:1994:MDP:528623}, where $\mathcal{X}$ and $\mathcal{A}$ denote state space and action space, $P$ denotes the transition probability given state and action, $R$ denotes state and action dependent reward function and $\gamma \in (0,1)$ denotes the reward discount factor.

For a policy $\pi$, define the discounted return sum a random variable by  $Z^\pi(x,a)=\sum_{t=0}^{\infty}\gamma^tR(x_t, a_t)$, where $x_0=x$, $a_0=a$, $x_{t} \sim P\left(\cdot | x_{t-1}, a_{t-1}\right)$ and $a_t \sim \pi(\cdot|x_t)$. The objective in reinforcement learning can be summarized as finding the optimal $\pi^*$ that maximizes the expectation of $Z^\pi$, the action-value function $Q^\pi(x, a)=\mathbb{E}[Z^\pi(x,a)]$. The most common approach is to find the unique fixed point of the Bellman optimality operator $\mathcal{T}$~\citep{Bellman:1957}: 
\begin{equation*}
    Q^*(x, a)=\mathcal{T} Q^*(x, a) :=\mathbb{E}[R(x, a)]+\gamma \mathbb{E}_{P} \max _{a^{\prime}} Q^*\left(x^{\prime}, a^{\prime}\right).
\end{equation*}

To update $Q$, which is approximated by a neural network in most deep reinforcement learning studies, $Q$-learning~\citep{watkins1989learning} iteratively trains the network by minimizing the squared temporal difference (TD) error defined by
\begin{equation*}
    \delta_{t}^{2}=\left[r_{t}+\gamma \max _{a^{\prime} \in \mathcal{A}} Q\left(x_{t+1}, a^{\prime}\right)-Q\left(x_{t}, a_{t}\right)\right]^{2}
\end{equation*}
along the trajectory observed while the agent interacts with the environment following $\epsilon$-greedy policy.
DQN~\citep{mnih2015humanlevel} uses a convolutional neural network to represent $Q$ and achieves human-level play on the Atari-57 benchmark.

\subsection{Distributional RL}
Instead of a scalar $Q^\pi(x, a)$, distributional RL looks into the intrinsic randomness of $Z^\pi$ by studying its distribution. The distributional Bellman operator for policy evaluation is 
\begin{equation*}
    Z^\pi(x, a) \stackrel{D}{=} R(x, a)+\gamma Z^\pi\left(X^{\prime}, A^{\prime}\right),
\end{equation*}
where $X^\prime \sim P(\cdot|x,a)$ 
and $A^\prime\sim\pi(\cdot|X^\prime)$,  $A\stackrel{D}{=}B$ denotes that random variable $A$ and $B$ follow the same distribution.

Both theory and algorithms have been established for distributional RL. In theory, the distributional Bellman operator for policy evaluation is proved to be a contraction in the $p$-Wasserstein distance~\citep{bellemare2017distributional}. \cite{bellemare2017distributional} shows that C51 outperforms value-based RL, in addition \cite{hessel2018rainbow} combined C51 with enhancements such as prioritized experience replay~\citep{schaul2016prioritized}, n-step updates~\citep{sutton1988learning}, and the dueling architecture~\citep{wang2016dueling}, leading to the Rainbow agent, current state-of-the-art in Atari-57 for non-distributed agents, while the distributed algorithm proposed by ~\cite{kapturowski2018recurrent} achieves state-of-the-art performance for all agents.
From an algorithmic perspective, it is impossible to represent the full space of probability distributions with a finite collection of parameters. Therefore the parameterization of quantile functions is usually the most crucial part in a general distributional RL algorithm.
In C51, the true distribution is projected to a categorical distribution~\citep{bellemare2017distributional} with fixed values for parameterization. QR-DQN fixes probabilities instead of values, and parameterizes the quantile values~\citep{dabney2018implicit} while IQN randomly samples the probabilities~\citep{dabney2018implicit}. We will introduce QR-DQN and IQN in Section 2.2, and extend from their work to ours.



\subsection{Quantile Regression for Distributional RL}


In contrast to C51 which estimates probabilities for $N$ fixed locations in return, QR-DQN~\citep{dabney2018distributional} estimates the respected quantile values for $N$ fixed, uniform probabilities. In QR-DQN, the distribution of the random return is approximated by a uniform mixture of $N$ Diracs,
\begin{equation*}
    Z_{\theta}(x, a):=\frac{1}{N}\sum_{i=1}^{N}\delta_{\theta_i(x, a)},
\end{equation*}
with each $\theta_i$ assigned a quantile value trained with quantile regression.

Based on QR-DQN, \cite{dabney2018implicit} propose using probabilities sampled from a base distribution, e.g. $\tau\in U([0,1])$, rather than fixed probabilities. They further learn the quantile function that maps from embeddings of sampled probabilities to the corresponding quantiles, called implicit quantile value network (IQN). 
At the time of this writing, IQN achieves the state-or-the-art performance on Atari-57 benchmark, human-normalized mean and median of all agents that does not combine distributed RL, prioritized replay~\citep{schaul2016prioritized} and $n$-step update.

\cite{dabney2018implicit} claimed that with enough network capacity, IQN is able to approximate to the full quantile function with infinite number of quantile fractions. However, in practice one needs to use a finite number of quantile fractions to estimate action values for decision making, e.g. 32 randomly sampled quantile fractions as in \cite{dabney2018implicit}. With limited fractions, a natural question arises that, how to best utilize those fractions to find the closest approximation of the true distribution?  

\section{Our Algorithm}
We propose Fully parameterized Quantile Function (FQF) for Distributional RL. Our algorithm consists of two networks, the fraction proposal network that generates a set of quantile fractions for each state-action pair, and the quantile value network that maps probabilities to quantile values. We first describe the fully parameterized quantile function in Section 3.1, with variables on both probability axis and value axis. Then, we show how to train the fraction proposal network in Section 3.2, and how to train the quantile value network with quantile regression in Section 3.3. Finally, we present our algorithm and describe the implementation details in Section 3.4. 

\subsection{Fully Parameterized Quantile Function}
In FQF, we estimate $N$ adjustable quantile values for $N$ adjustable quantile fractions to approximate the quantile function. The distribution of the return is approximated by a weighted mixture of $N$ Diracs given by
\begin{equation}
\label{project_distribution}
    Z_{\theta, \tau}(x, a):=\sum_{i=0}^{N-1}(\tau_{i+1} - \tau_{i})\delta_{\theta_i(x, a)},
\end{equation}
where $\delta_z$ denotes a Dirac at $z \in \mathbb{R}$, $\tau_1,...\tau_{N-1}$ represent the N-1 adjustable fractions satisfying $\tau_{i-1} < \tau_i$, with $\tau_0=0$ and $\tau_N=1$ to simplify notation. Denote quantile function~\citep{muller1997integral} $F^{-1}_Z$ the inverse function of cumulative distribution function $F_Z(z)=Pr(Z<z)$. By definition we have
\begin{equation*}
   F_{Z}^{-1}(p) :=\inf \left\{z \in \mathbb{R} : p \leq F_{Z}(z)\right\}
\end{equation*}
where $p$ is what we refer to as quantile fraction.

Based on the distribution in Eq.(\ref{project_distribution}), denote $\Pi^{\theta, \tau}$ the projection operator that projects quantile function onto a staircase function supported by $\theta$ and $\tau$, the projected quantile function is given by
\begin{equation*}
    F^{-1,\theta,\tau}_{Z}(\omega)=\Pi^{\theta, \tau} F^{-1}_{Z}(\omega)=
    \theta_0+\sum_{i=0}^{N-1} (\theta_{i+1}-\theta_{i})H_{\tau_{i+1}}(\omega),
\end{equation*}
where $H$ is the Heaviside step function and $H_\tau(\omega)$ is the short for $H(\omega-\tau)$. Figure~\ref{fig:proj} gives an example of such projection. For each state-action pair $(x,a)$, we first generate the set of fractions $\tau$ using the fraction proposal network, and then obtain the quantiles values $\theta$ corresponding to $\tau$ using the quantile value network. 

To measure the distortion between approximated quantile function and the true quantile function, we use the $1$-Wasserstein metric given by
\begin{equation}
\label{projection_error}
    W_{1}(Z, \theta, \tau)=\sum_{i=0}^{N-1} \int_{\tau_{i}}^{\tau_{i+1}}\left|F^{-1}_{Z}(\omega)-\theta_{i}\right| d \omega.
\end{equation}
Unlike KL divergence used in C51 which considers only the probabilities of the outcomes, the $p$-Wasseretein metric takes both the probability and the distance between outcomes into consideration. 
Figure~\ref{fig:proj} illustrates the concept of how different approximations could affect $W_1$ error, and shows an example of $\Pi_{W_1}$. However, note that in practice Eq.(\ref{projection_error}) can not be obtained without bias.

\begin{figure}[!htb]
   \begin{minipage}{0.48\textwidth}
     \centering
     \includegraphics[width=1.0\linewidth]{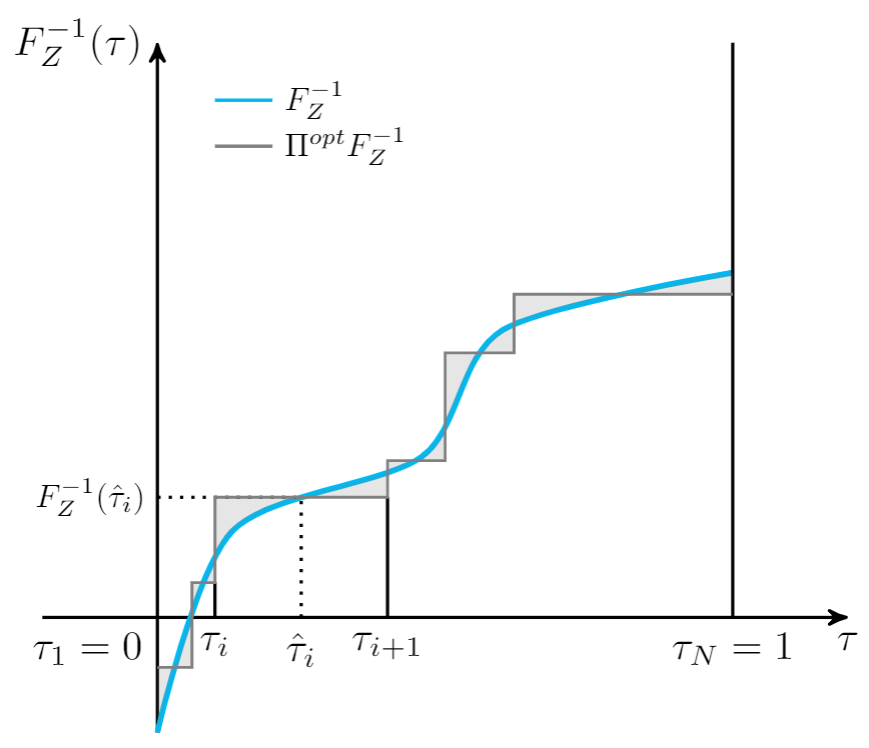}
     \caption*{(a)}
   \end{minipage}\hfill
   \begin{minipage}{0.48\textwidth}
     \centering
     \includegraphics[width=1.0\linewidth]{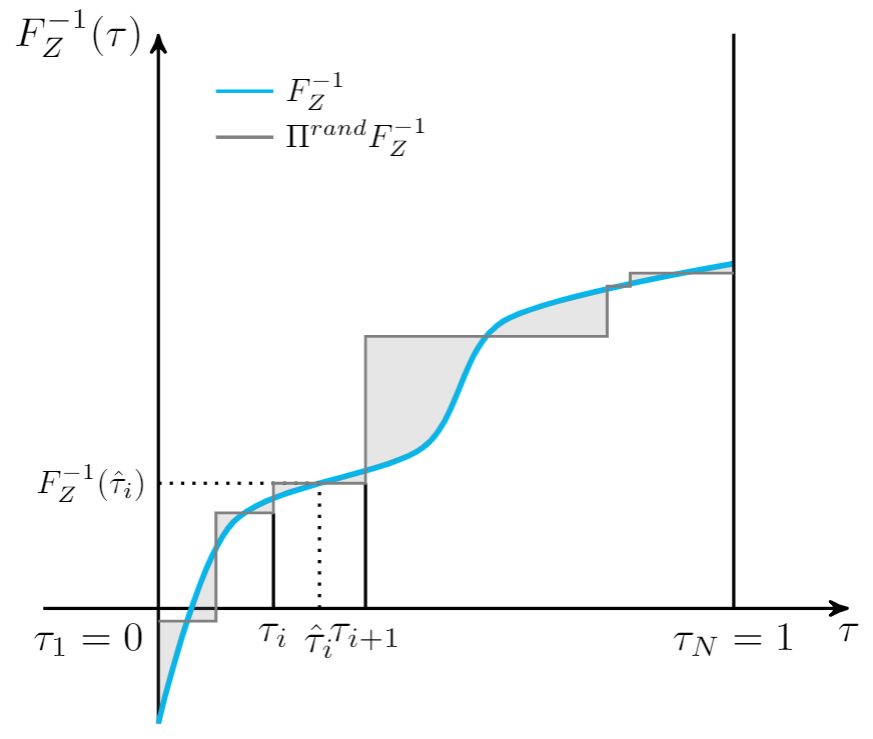}
     \caption*{(b)}
   \end{minipage}
   \caption{Two approximations of the same quantile function using different set of $\tau$ with $N=6$, the area of the shaded region is equal to the 1-Wasserstein error. (a) Finely-adjusted $\tau$ with minimized $W_1$ error. (b) Randomly chosen $\tau$ with larger $W_1$ error.}
   \label{fig:proj}
\end{figure}

\subsection{Training fraction proposal Network}
To achieve minimal $1$-Wasserstein error, we start from fixing $\tau$ and finding the optimal corresponding quantile values $\theta$. In QR-DQN, \cite{dabney2018implicit} gives an explicit form of $\theta$ to achieve the goal. We extend it to our setting:

\begin{lemma}
\label{lemma:1}
\citep{dabney2018implicit} For any $\tau_1,...\tau_{N-1} \in [0,1]$ satisfying $\tau_{i-1} < \tau_i$ for $i$, with $\tau_1=0$ and $\tau_N=1$, and cumulative distribution function $F$ with inverse $F^{-1}$, the set of $\theta$ minimizing Eq.(\ref{projection_error})
is given by 
\begin{equation}
\label{eq:1}
    \theta_{i}=F^{-1}_{Z}(\frac{\tau_i+\tau_{i+1}}{2})
\end{equation}
\end{lemma}

We can now substitute $\theta_i$ in Eq.(\ref{projection_error}) with equation Eq.(\ref{eq:1}) and find the optimal condition for $\tau$ to minimize $W_1(Z,\tau)$. For simplicity, we denote $\hat{\tau}_i=\frac{\tau_i+\tau_{i+1}}{2}$.

\begin{restatable}{proposition}{prop}
\label{proposition:1}
For any continuous quantile function $F^{-1}_{Z}$ that is non-decreasing, define the 1-Wasserstein loss of $F^{-1}_{Z}$ and $F^{-1,\tau}_{Z}$ by
\begin{equation}
\label{objective}
    W_{1}(Z, \tau)=\sum_{i=0}^{N-1} \int_{\tau_{i}}^{\tau_{i+1}}\left|F^{-1}_{Z}(\omega)-F^{-1}_{Z}(\hat{\tau}_i)\right| d \omega.
\end{equation}
$\frac{\partial W_{1}}{\partial \tau_i}$ is given by 
\begin{equation}
\label{eq:2}
    \frac{\partial W_{1}}{\partial \tau_i}=2F^{-1}_{Z}(\tau_{i})-F^{-1}_{Z}(\hat{\tau}_i)-F^{-1}_{Z}(\hat{\tau}_{i-1}), 
\end{equation}
$\forall i\in (0, N)$. 

Further more, $\forall \tau_{i-1},\; \tau_{i+1}\in[0,1],\; \tau_{i-1}<\tau_{i+1},\; \exists \tau_i \in (\tau_{i-1}, \tau_{i+1})\; s.t.\; \frac{\partial W_{1}}{\partial \tau_i}=0$.
\end{restatable}


Proof of proposition~\ref{proposition:1} is given in the appendix. While computing $W_1$ without bias is usually impractical, equation \ref{eq:2} provides us with a way to minimize $W_1$ without computing it. Let $w_1$ be the parameters of the fraction proposal network $P$, for an arbitrary quantile function $F_Z^{-1}$, we can minimize $W_1$ by iteratively applying gradients descent to $w_1$ according to Eq.(\ref{eq:2}) and convergence is guaranteed. As the true quantile function $F^{-1}_{Z}$ is unknown to us in practice, we use the quantile value network $F^{-1}_{Z,w_2}$ with parameters $w_2$ for current state and action as true quantile function. 

The expected return, also known as action-value based on FQF is then given by
\begin{equation*}
    Q(x, a)=\sum_{i=0}^{N-1}(\tau_{i+1}-\tau_{i})F^{-1}_{Z,w_2}(\hat{\tau}_i),
\end{equation*}
where $\tau_0=0$ and $\tau_{N}=1$.

\subsection{Training quantile value network}
With the properly chosen probabilities, we combine quantile regression and distributional Bellman update on the optimized probabilities to train the quantile function. Consider $Z$ a random variable denoting the action-value at $(x_t,a_t)$ and $Z'$ the action-value random variable at $(x_{t+1},a_{t+1})$, the weighted temporal difference (TD) error for two probabilities $\hat{\tau}_i$ and $\hat{\tau}_j$ is defined by
\begin{equation}
    \delta_{ij}^t=r_t+\gamma F^{-1}_{Z',w_1}(\hat{\tau}_i)-F^{-1}_{Z,w_1}(\hat{\tau}_j)
\end{equation}

Quantile regression is used in QR-DQN and IQN to stochastically adjust the quantile estimates so as to minimize the Wasserstein distance to a target distribution. We follow QR-DQN and IQN where quantile value networks are trained by minimizing the Huber quantile regression loss~\citep{huber:1964}, with threshold $\kappa$,
\begin{equation*}
    \rho_{\tau}^{\kappa}\left(\delta_{i j}\right)=\left|\tau-\mathbb{I}\left\{\delta_{i j}<0\right\}\right| \frac{\mathcal{L}_{\kappa}\left(\delta_{i j}\right)}{\kappa} \text{, with}
\end{equation*}
\begin{equation*}
    \mathcal{L}_{\kappa}\left(\delta_{i j}\right)=\left\{\begin{array}{ll}{\frac{1}{2} \delta_{i j}^{2},} & {\text { if }\left|\delta_{i j}\right| \leq \kappa} \\ {\kappa\left(\left|\delta_{i j}\right|-\frac{1}{2} \kappa\right),} & {\text { otherwise }}\end{array}\right.
\end{equation*}

The loss of the quantile value network is then given by
\begin{equation}
    \mathcal{L}(x_t, a_t, r_t, x_{t+1})=\frac{1}{N}\sum_{i=0}^{N-1}\sum_{j=0}^{N-1}\rho_{\hat{\tau}_j}^\kappa(\delta_{ij}^{t})
\end{equation}
Note that $F^{-1}_{Z}$ and its Bellman target share the same proposed quantile fractions $\hat{\tau}$ to reduce computation.

We perform joint gradient update for $w_1$ and $w_2$, as illustrated in Algorithm \ref{algo:1}.

\begin{algorithm}[H]
\SetAlgoLined
\label{algo:1}
\KwParameter{$N,\kappa$}
\KwInput{$x,a,r,x', \gamma \in[0,1)$}
 \tcp{Compute proposed fractions for $x,a$}
 $\tau\leftarrow P_{w_1}(x,a)$\;
  \tcp{Compute proposed fractions for $x',a'$}
 \For {$a'\in \mathcal{A}$}{
  $\tau^{a'}\leftarrow P_{w_1}(x',a')$\; 
 }
 \tcp{Compute greedy action}
 $Q(s',a')\leftarrow \sum_{i=0}^{N-1}(\tau^{a'}_{i+1}-\tau^{a'}_{i})F^{-1}_{Z',w_2}(\hat{\tau}_i^{a'})$\;
 $a^*\leftarrow \underset{a'}{\mathrm{argmax}} Q(s',a')$\; 
 \tcp{Compute $L$}
 \For {$0\leq i \leq N-1$}{
    \For {$0\leq j \leq N-1$}{
        $\delta_{ij}\leftarrow r+\gamma F^{-1}_{Z',w_2}(\hat{\tau}_i)-F^{-1}_{Z,w_2}(\hat{\tau}_j)$
    }
 }
 $\mathcal{L}=\frac{1}{N}\sum_{i=0}^{N-1}\sum_{j=0}^{N-1}\rho_{\hat{\tau}_j}^\kappa(\delta_{ij})$\;
 \tcp{Compute $\frac{\partial W_1}{\partial \tau_i}$ for $i\in [1, N-1]$}
 $\frac{\partial W_1}{\partial \tau_i}=2F^{-1}_{Z,w_2}(\tau_{i})-F^{-1}_{Z,w_2}(\hat{\tau}_i)-F^{-1}_{Z,w_2}(\hat{\tau}_{i-1})$\;
 Update $w_1$ with $\frac{\partial W_1}{\partial \tau_i}$; Update $w_2$ with $\nabla\mathcal{L}$\;
 \KwOutput{$Q$}
\caption{FQF update}
\end{algorithm}

\subsection{Implementation Details}
Our fraction proposal network is represented by one fully-connected MLP layer. It takes the state embedding of original IQN as input and generates fraction proposal. Recall that in Proposition~\ref{lemma:1}, we require $\tau_{i-1}<\tau_i$ and $\tau_0=0,\tau_N=1$. While it is feasible to have $\tau_0=0,\tau_{N}=1$ fixed and sort the output of $\tau_{w_1}$, the sort operation would make the network hard to train. A more reasonable and practical way would be to let the neural network automatically have the output sorted using cumulated softmax. Let $q\in\mathbb{R}^{N}$ denote the output of a softmax layer, we have $q_i\in(0,1), i\in[0,N-1]$ and $\sum_{i=0}^{N-1}q_i=1$. Let $\tau_i=\sum_{j=0}^{i-1}q_j, i\in[0,N]$, then straightforwardly we have $\tau_i<\tau_j$ for $\forall i<j$ and $\tau_0=0,\tau_{N}=1$ in our fraction proposal network. Note that as $W_1$ is not computed, we can't directly perform gradient descent for the fraction proposal network. Instead, we use the \verb|grad_ys| argument in the tensorflow operator \verb|tf.gradients| to assign $\frac{\partial W_1}{\partial \tau_i}$ to the optimizer. In addition, one can use entropy of $q$ as a regularization term $H(q)=-\sum_{i=0}^{N-1} q_i\log q_i$ to prevent the distribution from degenerating into a deterministic one.

We borrow the idea of implicit representations from IQN to our quantile value network. To be specific, we compute the embedding of $\tau$, denoted by $\phi(\tau)$, with 
\begin{equation*}
    \phi_{j}(\tau) :=\operatorname{ReLU}\left(\sum_{i=0}^{n-1} \cos (i\pi \tau) w_{i j}+b_{j}\right),
\end{equation*}
where $w_{ij}$ and $b_j$ are network parameters. We then compute the element-wise (Hadamard) product of state feature $\psi(x)$ and embedding $\phi(\tau)$. Let $\odot$ denote element-wise product, the quantile values are given by $F^{-1}_Z(\tau)\approx F^{-1}_{Z,w_2}(\psi(x)\odot \phi(\tau))$.

In IQN, after the set of $\tau$ is sampled from a uniform distribution, instead of using differences between $\tau$ as probabilities of the quantiles, the mean of the quantile values is used to compute action-value $Q$. While in expectation, $Q=\sum_{i=0}^{N-1}(\tau_{i+1}-\tau_{i})F^{-1}_{Z}(\frac{\tau_i+\tau_{i+1}}{2})$ with $\tau_0=0, \tau_{N}=1$ and $Q=\frac{1}{N}\sum_{i=1}^{N}F^{-1}_{Z}(\tau_i)$ are equal, we use the former one to consist with our projection operation.

\section{Experiments}

\begin{figure*}[t!]
    \centering
    \includegraphics[width=0.3\textwidth]{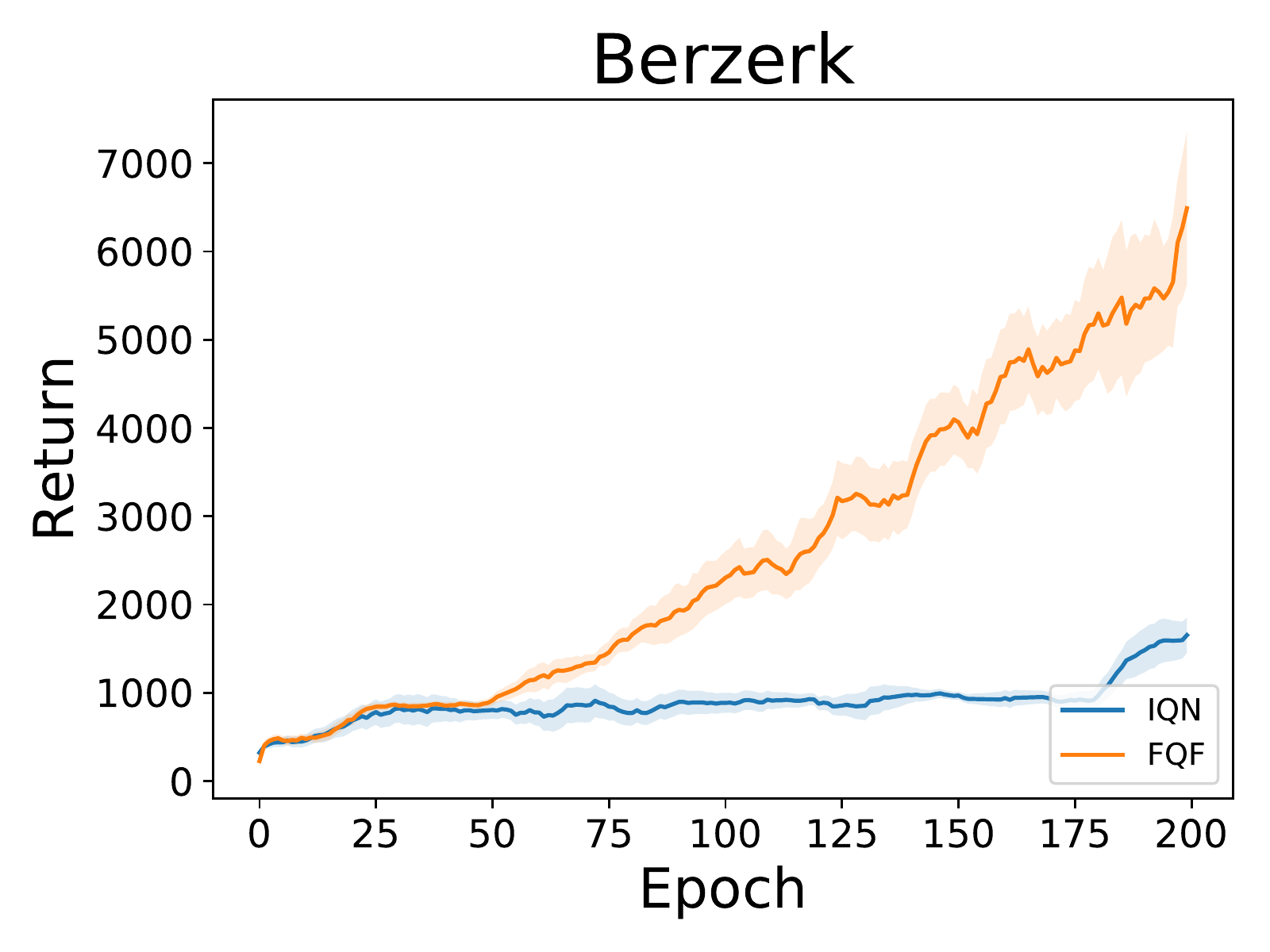}
    \includegraphics[width=0.3\textwidth]{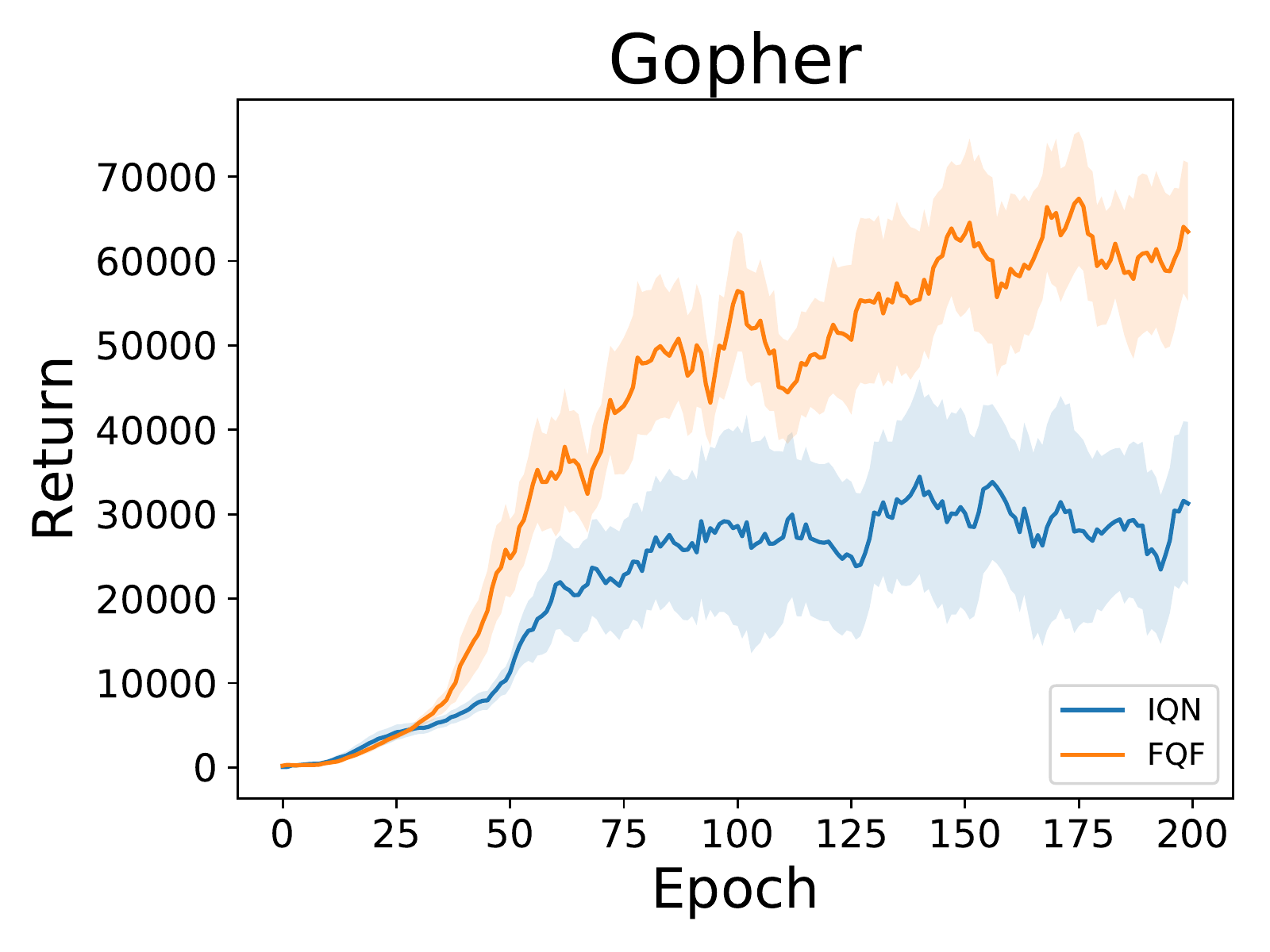}
    \includegraphics[width=0.3\textwidth]{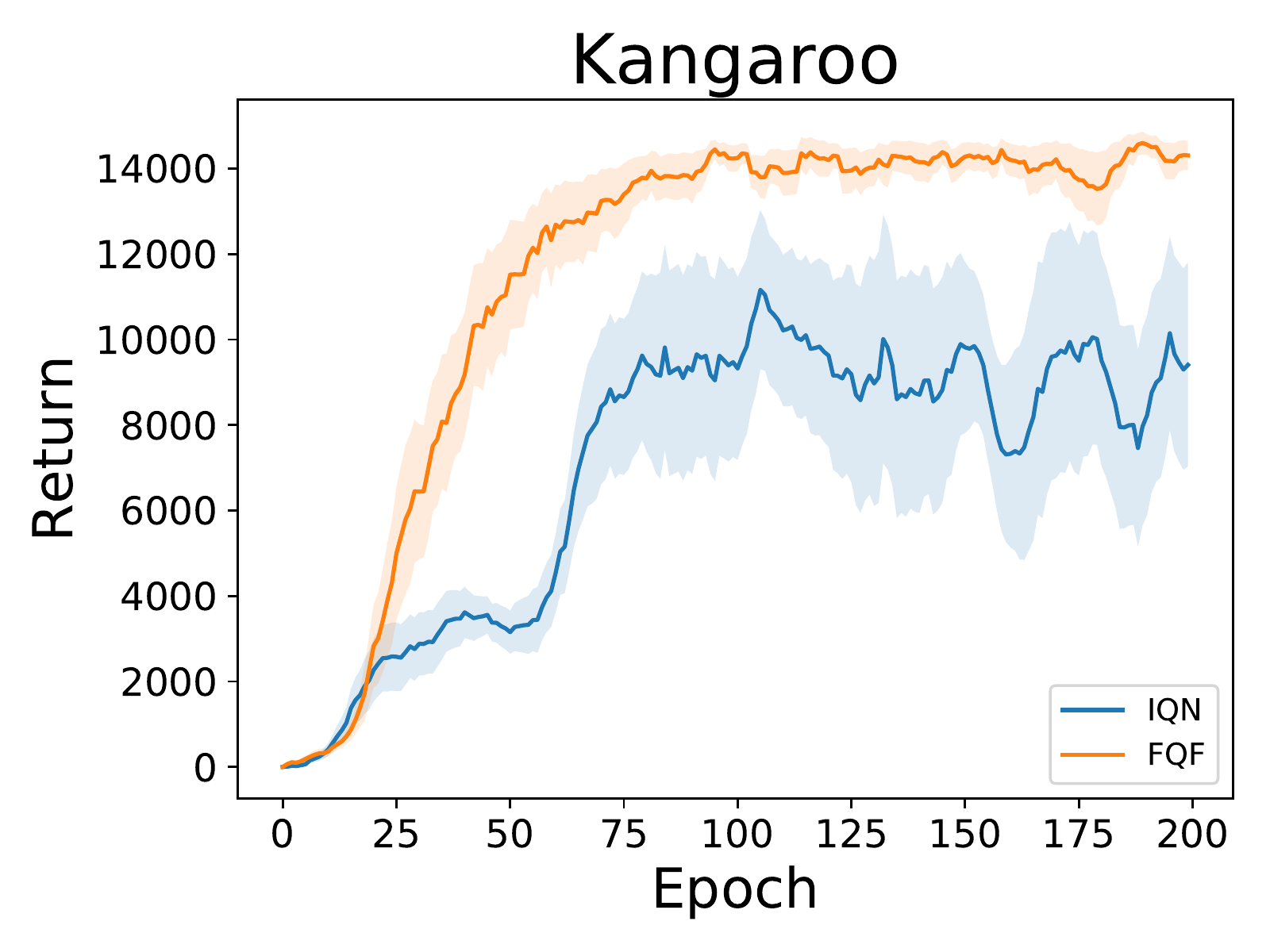}
    \includegraphics[width=0.3\textwidth]{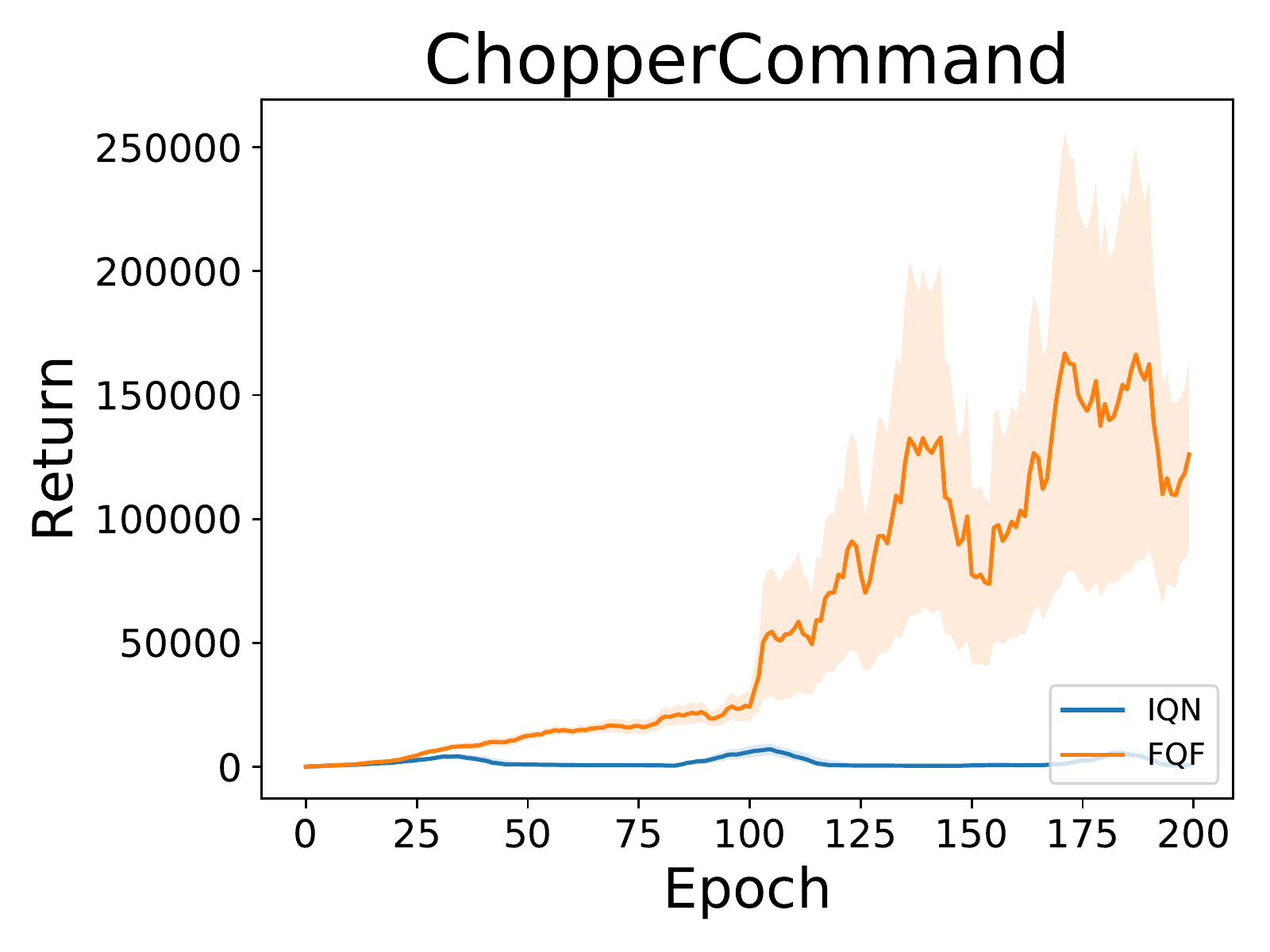}
    \includegraphics[width=0.3\textwidth]{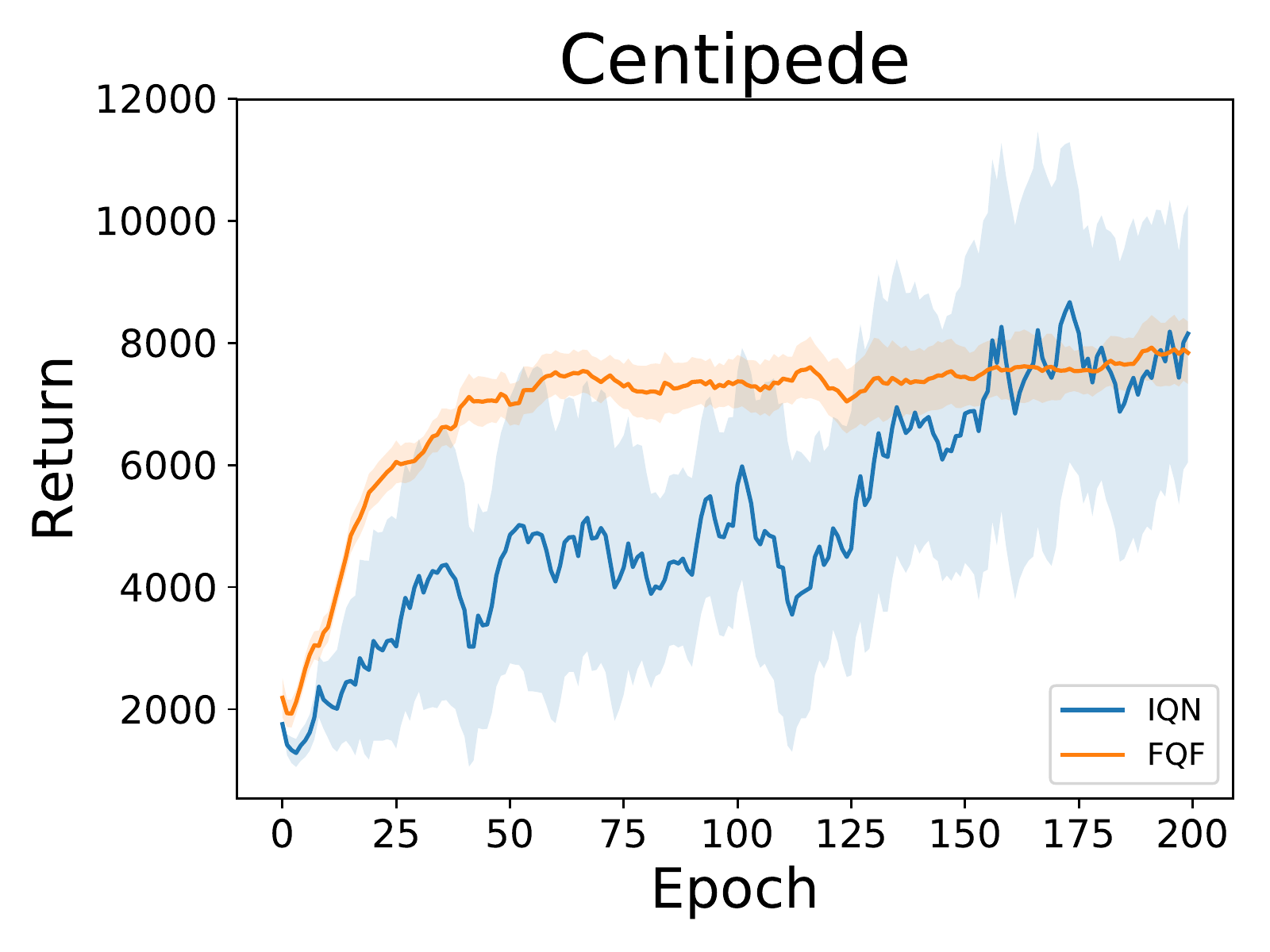}
    \includegraphics[width=0.3\textwidth]{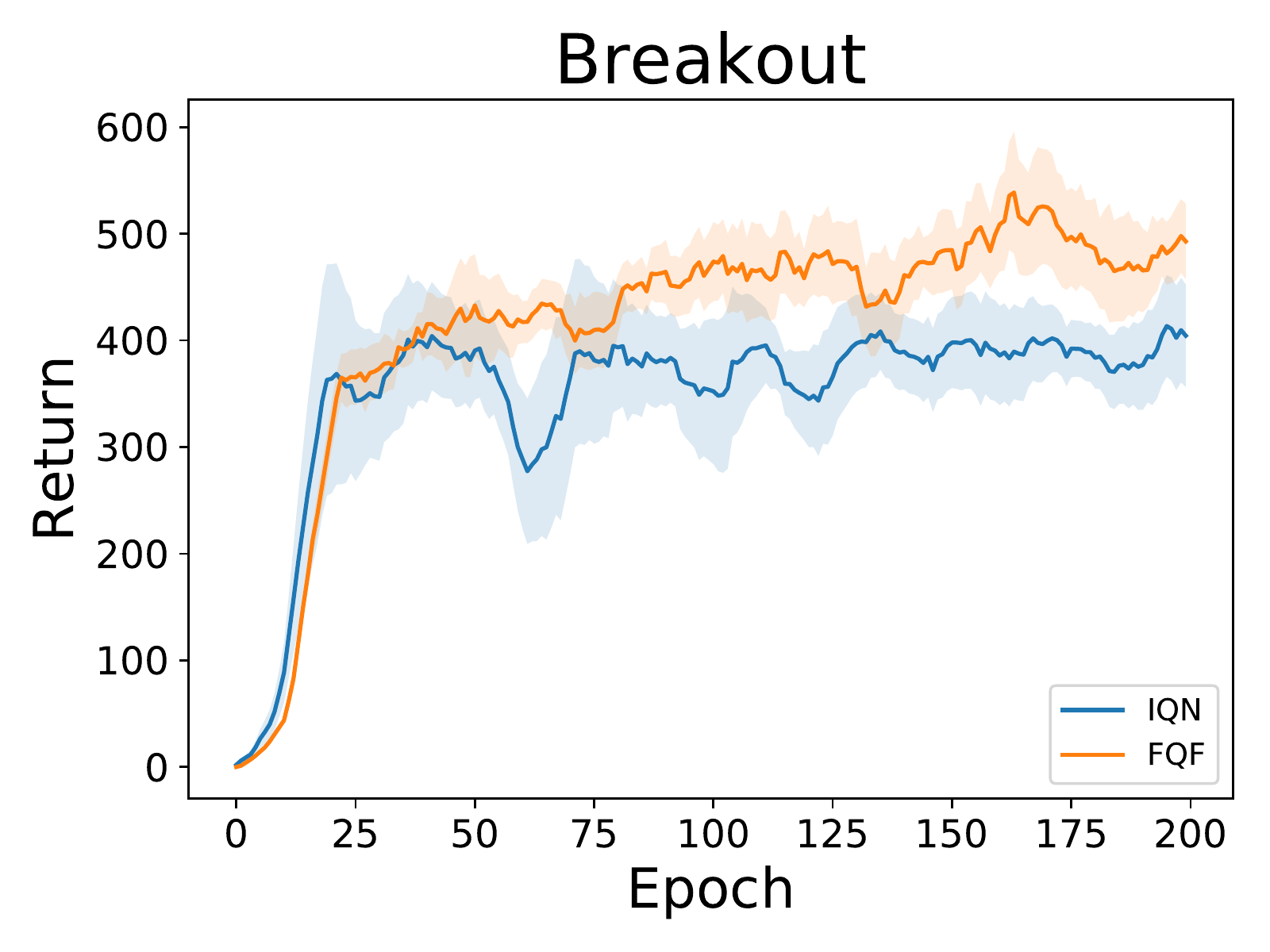}
    \includegraphics[width=0.3\textwidth]{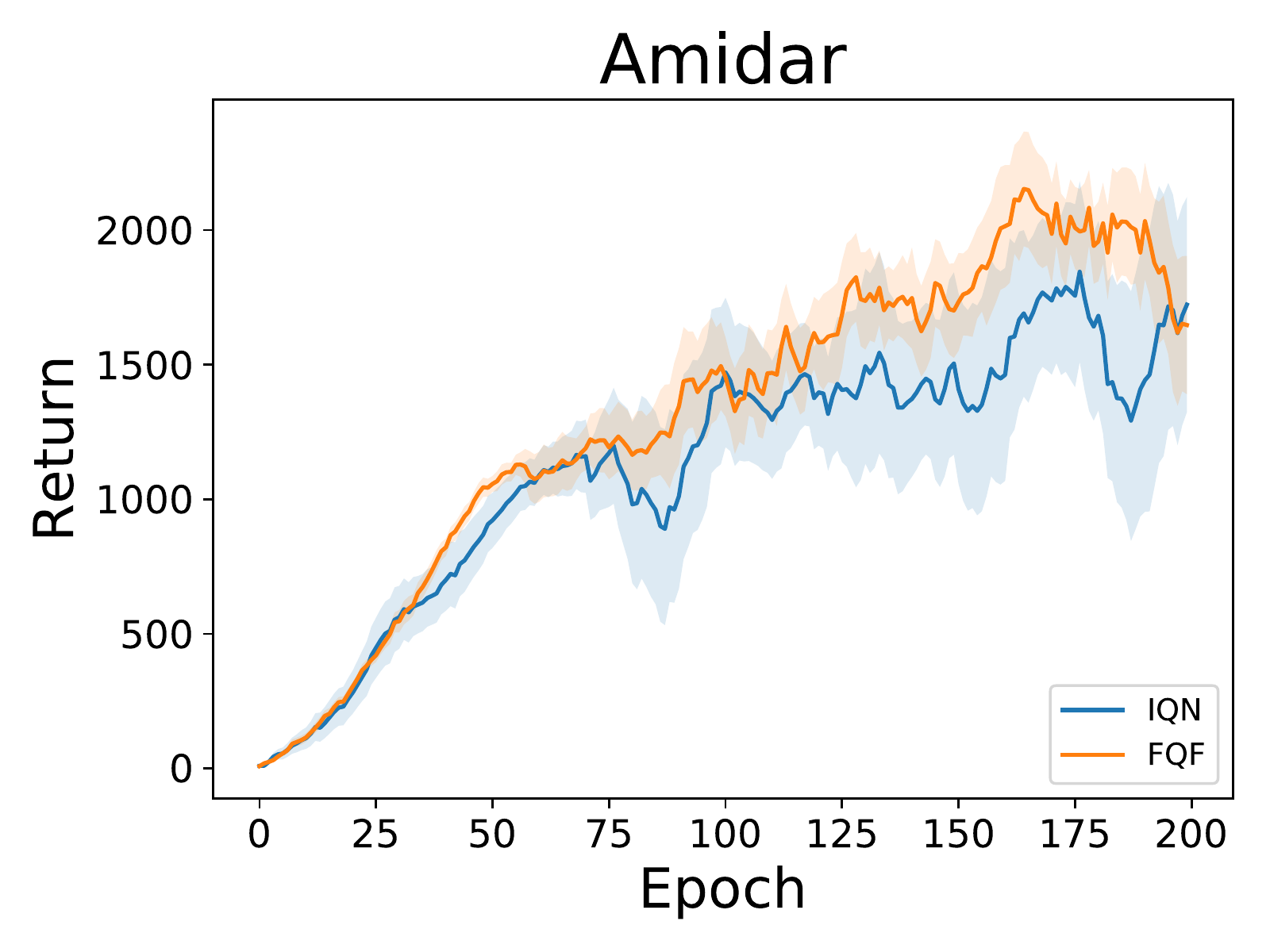}
    \includegraphics[width=0.3\textwidth]{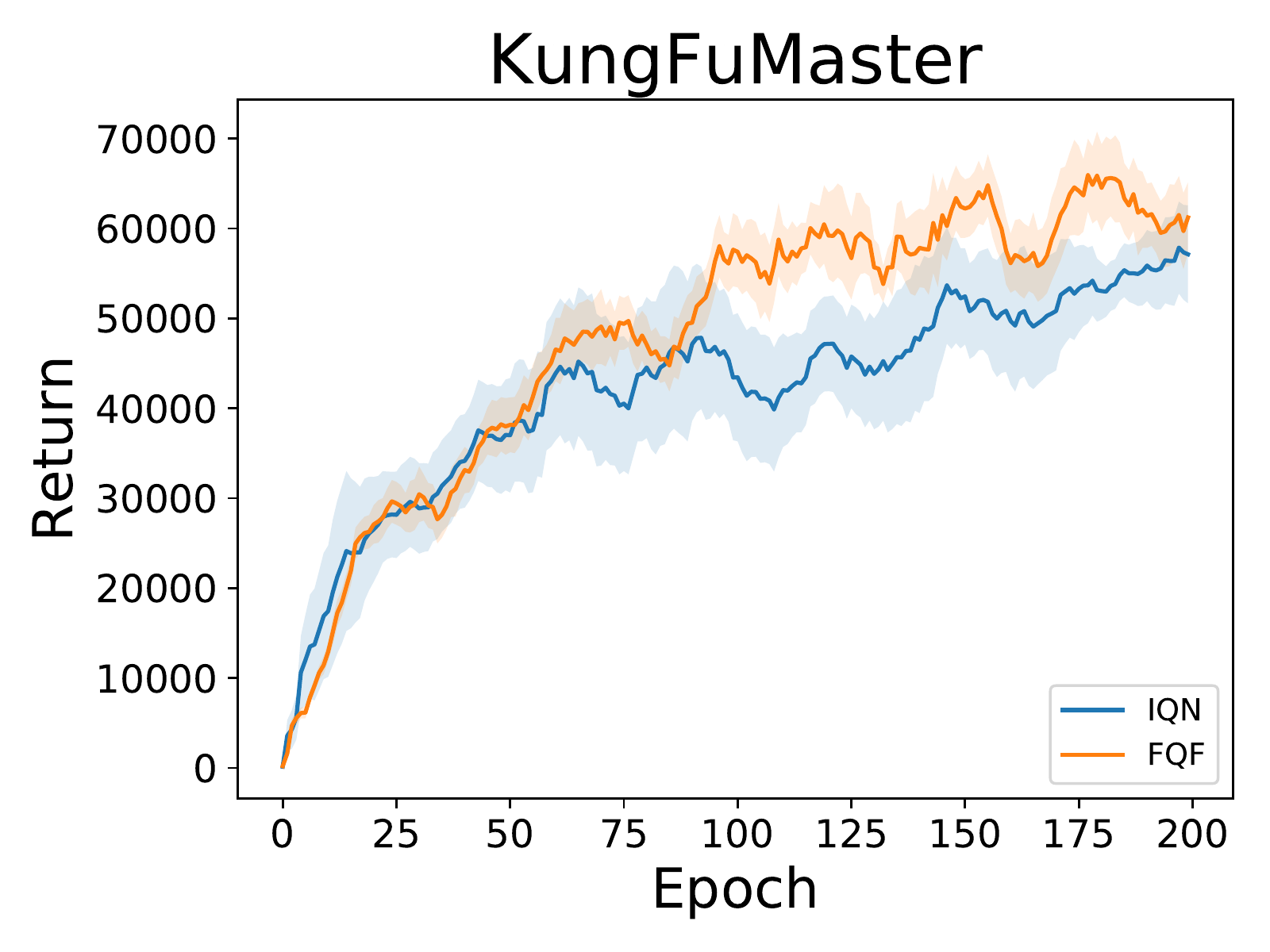}
    \includegraphics[width=0.3\textwidth]{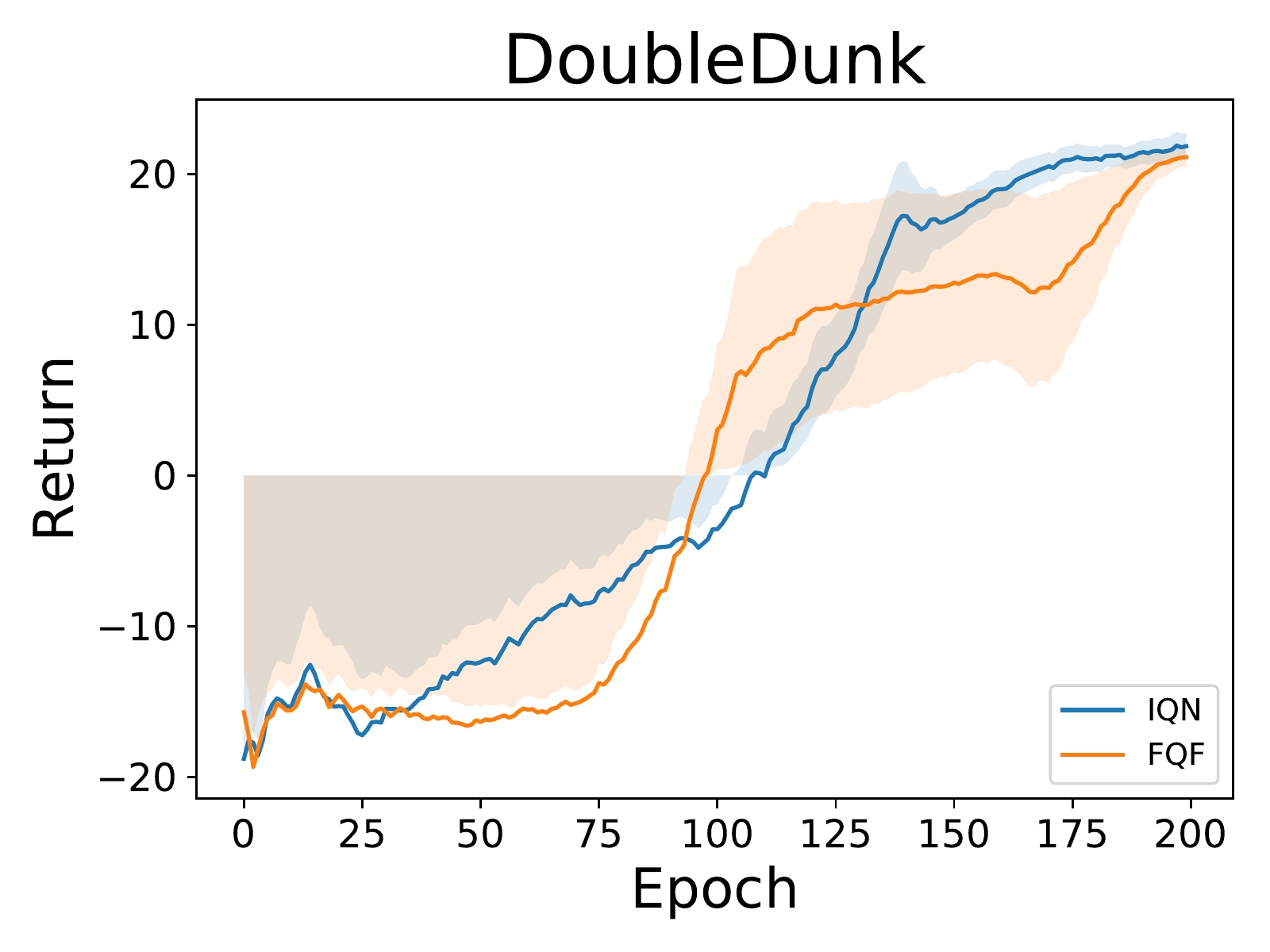}

    \caption{Performance comparison with IQN. Each training curve is averaged by 3 seeds. The training curves are smoothed with a moving average of 10 to improve readability.}
    \label{curve}
\end{figure*}

We test our algorithm on the Atari games from Arcade Learning Environment (ALE)~\cite{bellemare2013arcade}. We select the most relative algorithm to ours, IQN~\citep{dabney2018implicit}, as baseline, and compare FQF with QR-DQN~\citep{dabney2018distributional}, C51~\citep{bellemare2017distributional}, prioritized experience replay~\citep{schaul2016prioritized} and Rainbow~\citep{hessel2018rainbow}, the current state-of-art that combines the advantages of several RL algorithms including distributional RL. The baseline algorithm is implemented by \cite{castro18dopamine} in the Dopamine framework, with slightly lower performance than reported in IQN. We implement FQF based on the Dopamine framework. Unfortunately, we fail to test our algorithm on \textit{Surround} and \textit{Defender} as \textit{Surround} is not supported by the Dopamine framework and scores of \textit{Defender} is unreliable in Dopamine. Following the common practice~\citep{van2016deep}, we use the 30-noop evaluation settings to align with previous works. Results of FQF and IQN using sticky action for evaluation proposed by \cite{machado2018revisiting} are also provided in the appendix. In all, the algorithms are tested on 55 Atari games.

Our hyper-parameter setting is aligned with IQN for fair comparison. The number of $\tau$ for FQF is 32. The weights of the fraction proposal network are initialized so that initial probabilities are uniform as in QR-DQN, also the learning rates are relatively small compared with the quantile value network to keep the probabilities relatively stable while training. We run all agents with 200 million frames. At the training stage, we use $\epsilon$-greedy with $\epsilon=0.01$. For each evaluation stage, we test the agent for 0.125 million frames with $\epsilon=0.001$. For each algorithm we run 3 random seeds. All experiments are performed on NVIDIA Tesla V100 16GB graphics cards. 

\begin{table*}[ht!]
	\centering
    \begin{tabular}{llllllll}
		\hline \hline
		& Mean & Median & >Human & >DQN \\
		\midrule
		DQN & 221\% & 79\% & 24 & 0 \\
		PRIOR. & 580\% & 124\% & 39 & 48 \\
		C51 & 701\% & 178\% & 40 & 50 \\
		RAINBOW & 1213\% & 227\% & 42 & 52\\
		QR-DQN & 902\% & 193\% & 41 & 54\\
		IQN & 1112\% & 218\%  & 39 & 54 \\
		\bottomrule
		FQF & \textbf{1426}\% & \textbf{272}\%  & \textbf{44} & \textbf{54}
	\end{tabular}
\caption{Mean and median scores across 55 Atari 2600 games, measured as percentages of human baseline. Scores are averages over 3 seeds.} \label{scores-table}
\end{table*}

Table~\ref{scores-table} compares the mean and median human normalized scores across 55 Atari games with up to 30 random no-op starts, and the full score table is provided in the Appendix. It shows that FQF outperforms all existing distributional RL algorithms, including Rainbow~\citep{hessel2018rainbow} that combines C51 with prioritized replay, and n-step updates. We also set a new record on the number of games where non-distributed RL agent performs better than human. 

Figure~\ref{curve} shows the training curves of several Atari games. Even on games where FQF and IQN have similar performance such as \textit{Centipede} , FQF is generally much faster thanks to self-adjusting fractions. 

However, one side effect of the full parameterization in FQF is that the training speed is decreased. With same settings, FQF is roughly 20\% slower than IQN due to the additional fraction proposal network. As the number of $\tau$ increases, FQF slows down significantly while IQN's training speed is not sensitive to the number of $\tau$ samples.

\section{Discussion and Conclusions}
Based on previous works of distributional RL, we propose a more general complete approximation of the return distribution. Compared with previous distributional RL algorithms, FQF focuses not only on learning the target, e.g. probabilities for C51, quantile values for QR-DQN and IQN, but also which target to learn, i.e quantile fraction. This allows FQF to learn a better approximation of the true distribution under restrictions of network capacity. Experiment result shows that FQF does achieve significant improvement.

There are some open questions we are yet unable to address in this paper. We will have some discussions here. First, does the $1$-Wasserstein error converge to its minimal value when the quantile function is not fixed? We cannot guarantee convergence of the fraction proposal network in deep neural networks where we involve quantile regression and Bellman update. Second, though we empirically believe so, does the contraction mapping result for fixed probabilities given by \cite{dabney2018distributional} also apply on self-adjusting probabilities? Third, while FQF does provide potentially better distribution approximation with same amount of fractions, how will a better approximated distribution affect agent's policy and how will it affect the training process? More generally, how important is quantile fraction selection during training?

As for future work, we believe that studying the trained quantile fractions will provide intriguing results. Such as how sensitive are the quantile fractions to state and action, and that how the quantile fractions will evolve in a single run. Also, the combination of distributional RL and DDPG in D4PG~\citep{barth2018distributed} showed that distributional RL can also be extended to continuous control settings. Extending our algorithm to continuous settings is another interesting topic. Furthermore, in our algorithm we adopted the concept of selecting the best target to learn. Can this intuition be applied to areas other than RL?

Finally, we also noticed that most of the games we fail to reach human-level performance involves complex rules that requires exploration based policies, such as \textit{Montezuma Revenge} and \textit{Venture}. Integrating distributional RL will be another potential direction as in \citep{tang2018exploration}. In general, we believe that our algorithm can be viewed as a natural extension of existing distributional RL algorithms, and that distributional RL may integrate greatly with other algorithms to reach higher performance.

\bibliographystyle{plainnat}
\bibliography{references}
\newpage
\input{appendix}
\end{document}

%% file: appendix.tex
\section*{Appendix}
\subsection*{Proof for proposition~\ref{proposition:1}}
\prop*
 
\begin{proof}
Note that $F^{-1}_Z$ is non-decreasing. We have
\begin{equation*}
    \begin{split}
        \frac{\partial W_1}{\partial \tau_i}=&\frac{\partial}{\partial \tau_i}(\int_{\tau_{i-1}}^{\tau_i} \left|F^{-1}_{Z}(\omega)-F^{-1}_{Z}(\hat{\tau}_{i-1})\right| d \omega + \int_{\tau_{i}}^{\tau_{i+1}} \left|F^{-1}_{Z}(\omega)-F^{-1}_{Z}(\hat{\tau}_{i})\right| d \omega)\\
        =&\frac{\partial}{\partial \tau_i}(\int_{\tau_{i-1}}^{\hat{\tau}_{i-1}} F^{-1}_{Z}(\hat{\tau}_{i-1})-F^{-1}_{Z}(\omega) d \omega + 
        \int_{\hat{\tau}_{i-1}}^{\tau_{i}} F^{-1}_{Z}(\omega)-F^{-1}_{Z}(\hat{\tau}_{i-1}) d \omega +\\
        &
        \int_{\tau_{i}}^{\tau_{i+1}} \left|F^{-1}_{Z}(\omega)-F^{-1}_{Z}(\hat{\tau}_{i})\right| d \omega))\\
        =&\frac{\tau_i-\tau_{i-1}}{4}\frac{\partial}{\partial \tau_i}F^{-1}_Z(\hat{\tau}_{i-1})+F^{-1}_Z(\tau_i)-F^{-1}_Z(\hat{\tau}_{i-1})-\frac{\tau_i-\tau_{i-1}}{4}\frac{\partial}{\partial \tau_i}F^{-1}_Z(\hat{\tau}_{i-1})+\\
        &\frac{\partial}{\partial \tau_i}(\int_{\tau_{i}}^{\tau_{i+1}} \left|F^{-1}_{Z}(\omega)-F^{-1}_{Z}(\hat{\tau}_{i})\right| d \omega))\\
        =&F^{-1}_Z(\tau_i)-F^{-1}_Z(\hat{\tau}_{i-1})+\frac{\partial}{\partial \tau_i}(\int_{\tau_{i}}^{\tau_{i+1}} \left|F^{-1}_{Z}(\omega)-F^{-1}_{Z}(\hat{\tau}_{i})\right| d \omega))\\
        =&F^{-1}_Z(\tau_i)-F^{-1}_Z(\hat{\tau}_{i-1})+F^{-1}_Z(\tau_i)-F^{-1}_Z(\hat{\tau}_{i})\\
        =&2F^{-1}_Z(\tau_i)-F^{-1}_Z(\hat{\tau}_{i-1})-F^{-1}_Z(\hat{\tau}_{i})
    \end{split}
\end{equation*}

As $F^{-1}_Z$ is non-decreasing we have $\frac{\partial W_1}{\partial \tau_i}|_{\tau_i=\tau_{i-1}}\leq0$ and $\frac{\partial W_1}{\partial \tau_i}|_{\tau_i=\tau_{i+1}}\geq0$. Recall that $F^{-1}_Z$ is continuous, so $\exists \tau_i \in (\tau_{i-1}, \tau_{i+1}) \ s.t.\ \frac{\partial W_1}{\partial \tau_i}=0$.
\end{proof}

\newpage

\subsection*{Hyper-parameter sheet}
\begin{table*}[ht]
	\centering
    \begin{tabular}{l|l|l}
Hyper-parameter & IQN & FQF \\ \hline
Learning rate & 0.00005 & 0.00005 \\
Optimizer & Adam & Adam \\
Batch size & 32 & 32 \\
Discount factor & 0.99 & 0.99 \\
Fraction proposal network learning rate & None & 2.5e-9 \\
Fraction proposal network optimizer & None & RMSProp \\
	\end{tabular}
\caption{hyper-parameter list}  \label{hyperparameter}
\end{table*}
We sweep the learning rate of fraction proposal network among (0, 2.5e-5) and finally fix this learning rate as 2.5e-9. For the training of fraction proposal network, we use RMSProp optimizer. Note that though the fraction proposal network takes the state embedding of original IQN as input, we only apply gradient to our new introduced parameter and do not back-propagate the gradient to the convolution layers.

\subsection*{Approximation demonstration}
To demonstrate how FQF provides a better quantile function approximation, figure \ref{fig:toy} provides plots of a toy case with different distributional RL algorithm's approximation of a known quantile function, from which we can see how quantile fraction selection affects distribution approximation.

\begin{figure}[!htb]
   \begin{minipage}{0.48\textwidth}
     \centering
     \includegraphics[width=1.0\linewidth]{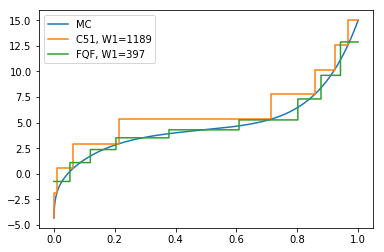}
     \caption*{(a)}
   \end{minipage}\hfill
   \begin{minipage}{0.48\textwidth}
     \centering
     \includegraphics[width=1.0\linewidth]{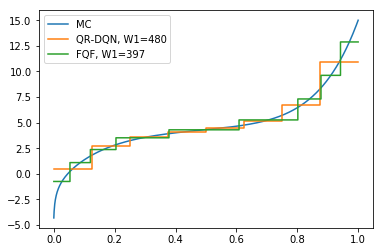}
     \caption*{(b)}
   \end{minipage}
   \caption{Demonstration of quantile function approximation on a toy case. $W_1$ denotes 1-Wasserstein distance between the approximated function and the one obtained through MC method.}
   \label{fig:toy}
\end{figure}

\subsection*{Varying number of quantile fractions}
Table \ref{N-table} gives mean scores of FQF and IQN over 6 Atari games, using different number of quantile fractions, i.e. $N$. For IQN, the selection of $N'$ is based on the highest score of each column given in \textit{Figure 2} of \citep{dabney2018implicit}. 

\begin{table*}[ht!]
	\centering
    \begin{tabular}{llllllll}
		\hline \hline
		& N=8 & N=32 & N=64\\
		\midrule
		IQN & 60.2 & 91.5 & 64.4\\
		FQF & 83.2 & 124.6 & 69.5
	\end{tabular}
\caption{Mean scores across 6 Atari 2600 games, measured as percentages of human baseline. Scores are averages over 3 seeds.} \label{N-table}
\end{table*}

Intuitively, the advantage of trained quantile fractions compared to random ones will be more observable at smaller $N$. At larger $N$ when both trained quantile fractions and random ones are densely distributed over $[0, 1]$, the differences between FQF and IQN becomes negligible. However from table~\ref{N-table} we see that even at large $N$, FQF performs slightly better than IQN.

\subsection*{Visualizing proposed quantile fraction}
In figure~\ref{demo}, we select a half-trained \textit{Kungfu Master} agent with $N=8$ to provide a case study of FQF. The reason why we choose a half-trained agent instead of a fully-trained agent is so that the distribution of $Q$ is not a deterministic one. Note that theoretically the quantile function should be non-decreasing, however from the example we can see that the learned quantile function might not always follow this property, and this phenomenon further motivates a quite interesting future work that leverages the non-decreasing property as prior knowledge for quantile function learning. The figure shows how the interval between proposed quantile fractions (i.e., the output of the softmax layer that sums to 1. See Section 3.4 for details) vary during a single run.

\begin{figure*}[htb!]
    \centering
    \includegraphics[width=1.0\textwidth]{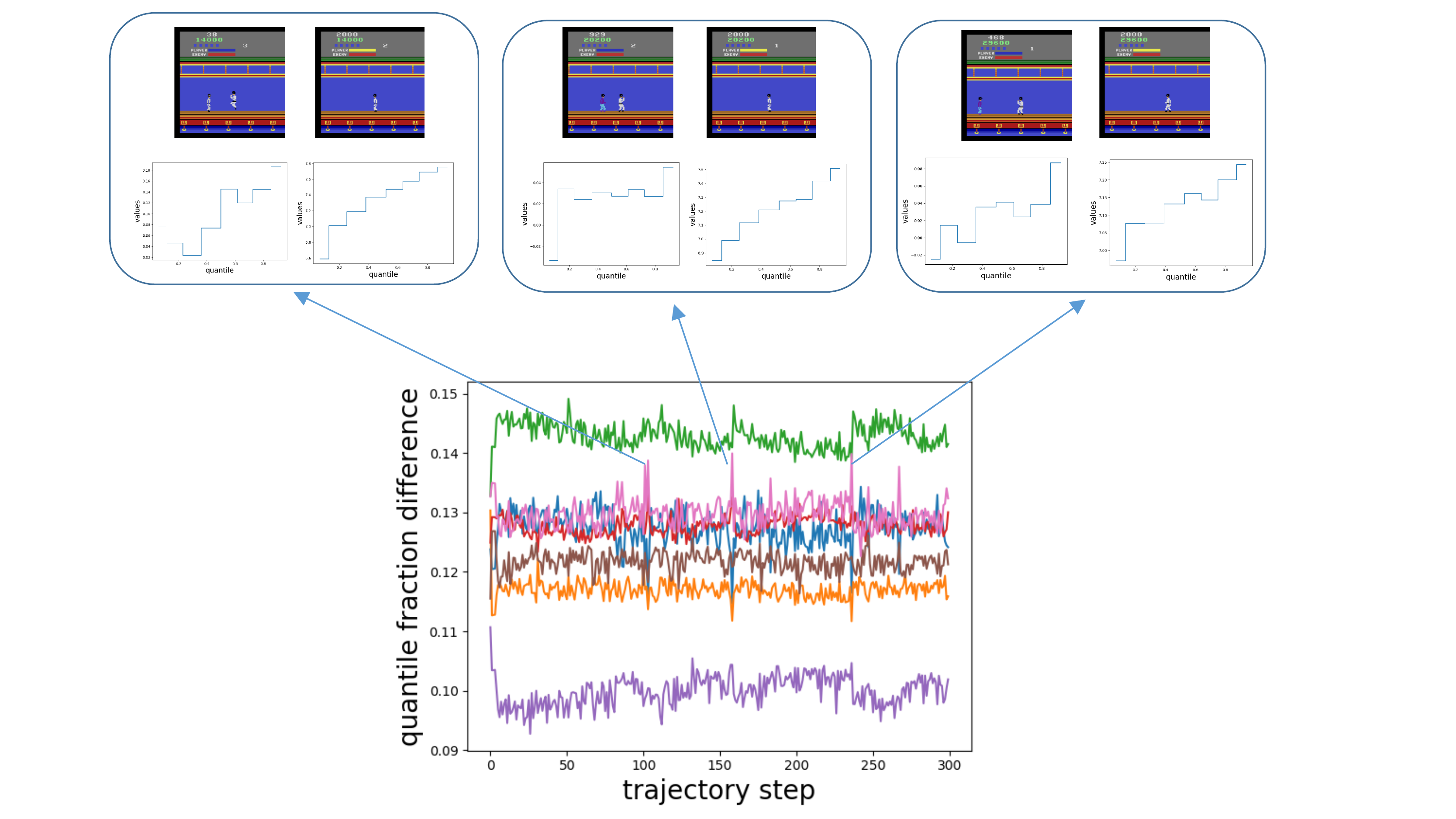}
    \caption{Interval between adjacent proposed quantile fractions for states at each time step in a single run. Different colors refer to different adjacent fractions' intervals, e.g. green curve refers to $\tau_2-\tau_1$.}
    \label{demo}
\end{figure*}

Whenever there appears an enemy behind the character, we see a spike in the fraction interval, indicating that proposed fraction is very different from that of following states without enemies. This suggests that the fraction proposal network is indeed state dependent and is able to provide different quantile fractions accordingly.

\newpage
\subsection*{ALE Scores}
\begin{table*}[ht]
\scriptsize
	\centering
    \begin{tabular}{l|l|l|l|l|l|l|l}
GAMES & RANDOM & HUMAN & DQN & PRIOR.DUEL. & QR-DQN & IQN & FQF \\ \hline
Alien & 227.8 & 7127.7 & 1620.0 & 3941.0 & 4871.0 & 7022.0 & \textbf{16754.6}\\ 
Amidar & 5.8 & 1719.5 & 978.0 & 2296.8 & 1641.0 & 2946.0 & \textbf{3165.3}\\ 
Assault & 222.4 & 742.0 & 4280.4 & 11477.0 & 22012.0 & \textbf{29091.0} & 23020.1\\ 
Asterix & 210.0 & 8503.3 & 4359.0 & 375080.0 & 261025.0 & 342016.0 & \textbf{578388.5}\\ 
Asteroids & 719.1 & 47388.7 & 1364.5 & 1192.7 & 4226.0 & 2898.0 & \textbf{4553.0}\\ 
Atlantis & 12850.0 & 29028.1 & 279987.0 & 395762.0 & 971850.0 & \textbf{978200.0} & 957920.0\\ 
BankHeist & 14.2 & 753.1 & 455.0 & \textbf{1503.1} & 1249.0 & 1416.0 & 1259.1\\ 
BattleZone & 2360.0 & 37187.5 & 29900.0 & 35520.0 & 39268.0 & 42244.0 & \textbf{87928.6}\\ 
BeamRider & 363.9 & 16926.5 & 8627.5 & 30276.5 & 34821.0 & \textbf{42776.0} & 37106.6\\ 
Berzerk & 123.7 & 2630.4 & 585.6 & 3409.0 & 3117.0 & 1053.0 & \textbf{12422.2}\\ 
Bowling & 23.1 & 160.7 & 50.4 & 46.7 & 77.2 & 86.5 & \textbf{102.3}\\ 
Boxing & 0.1 & 12.1 & 88.0 & 98.9 & \textbf{99.9} & 99.8 & 98.0\\ 
Breakout & 1.7 & 30.5 & 385.5 & 366.0 & 742.0 & 734.0 & \textbf{854.2}\\ 
Centipede & 2090.9 & 12017.0 & 4657.7 & 7687.5 & \textbf{12447.0} & 11561.0 & 11526.0\\ 
ChopperCommand & 811.0 & 7387.8 & 6126.0 & 13185.0 & 14667.0 & 16836.0 & \textbf{876460.0}\\ 
CrazyClimber & 10780.5 & 35829.4 & 110763.0 & 162224.0 & 161196.0 & 179082.0 & \textbf{223470.6}\\ 
DemonAttack & 152.1 & 1971.0 & 12149.4 & 72878.6 & 121551.0 & 128580.0 & \textbf{131697.0}\\ 
DoubleDunk & -18.6 & -16.4 & -6.6 & -12.5 & 21.9 & 5.6 & \textbf{22.9}\\ 
Enduro & 0.0 & 860.5 & 729.0 & 2306.4 & 2355.0 & 2359.0 & \textbf{2370.8}\\ 
FishingDerby & -91.7 & -38.7 & -4.9 & 41.3 & 39.0 & 33.8 & \textbf{52.7}\\ 
Freeway & 0.0 & 29.6 & 30.8 & 33.0 & \textbf{34.0} & \textbf{34.0} & 33.7\\ 
Frostbite & 65.2 & 4334.7 & 797.4 & 7413.0 & 4384.0 & 4324.0 & \textbf{16472.9}\\ 
Gopher & 257.6 & 2412.5 & 8777.4 & 104368.2 & 113585.0 & 118365.0 & \textbf{121144.0}\\ 
Gravitar & 173.0 & 3351.4 & 473.0 & 238.0 & 995.0 & 911.0 & \textbf{1406.0}\\ 
Hero & 1027.0 & 30826.4 & 20437.8 & 21036.5 & 21395.0 & 28386.0 & \textbf{30926.2}\\ 
IceHockey & -11.2 & 0.9 & -1.9 & -0.4 & -1.7 & 0.2 & \textbf{17.3}\\ 
Jamesbond & 29.0 & 302.8 & 768.5 & 812.0 & 4703.0 & 35108.0 & \textbf{87291.7}\\ 
Kangaroo & 52.0 & 3035.0 & 7259.0 & 1792.0 & 15356.0 & \textbf{15487.0} & 15400.0\\ 
Krull & 1598.0 & 2665.5 & 8422.3 & 10374.0 & \textbf{11447.0} & 10707.0 & 10706.8\\ 
KungFuMaster & 258.5 & 22736.3 & 26059.0 & 48375.0 & 76642.0 & 73512.0 & \textbf{111138.5}\\ 
MontezumaRevenge & 0.0 & 4753.3 & 0.0 & 0.0 & 0.0 & 0.0 & 0.0\\ 
MsPacman & 307.3 & 6951.6 & 3085.6 & 3327.3 & 5821.0 & 6349.0 & \textbf{7631.9}\\ 
NameThisGame & 2292.3 & 8049.0 & 8207.8 & 15572.5 & 21890.0 & \textbf{22682.0} & 16989.4\\ 
Phoenix & 761.4 & 7242.6 & 8485.2 & 70324.3 & 16585.0 & 56599.0 & \textbf{174077.5}\\ 
Pitfall & -229.4 & 6463.7 & -286.1 & 0.0 & 0.0 & 0.0 & 0.0\\ 
Pong & -20.7 & 14.6 & 19.5 & 20.9 & \textbf{21.0} & \textbf{21.0} & \textbf{21.0}\\ 
PrivateEye & 24.9 & 69571.3 & 146.7 & 206.0 & \textbf{350.0} & 200.0 & 140.1\\ 
Qbert & 163.9 & 13455.0 & 13117.3 & 18760.3 & \textbf{572510.0} & 25750.0 & 27524.4\\ 
Riverraid & 1338.5 & 17118.0 & 7377.6 & 20607.6 & 17571.0 & 17765.0 & \textbf{23560.7}\\ 
RoadRunner & 11.5 & 7845.0 & 39544.0 & 62151.0 & \textbf{64262.0} & 57900.0 & 58072.7\\ 
Robotank & 2.2 & 11.9 & 63.9 & 27.5 & 59.4 & 62.5 & \textbf{75.7}\\ 
Seaquest & 68.4 & 42054.7 & 5860.6 & 931.6 & 8268.0 & \textbf{30140.0} & 29383.3\\ 
Skiing & -17098.1 & -4336.9 & -13062.3 & -19949.9 & -9324.0 & -9289.0 & \textbf{-9085.3}\\ 
Solaris & 1236.3 & 12326.7 & 3482.8 & 133.4 & 6740.0 & \textbf{8007.0} & 6906.7\\ 
SpaceInvaders & 148.0 & 1668.7 & 1692.3 & 15311.5 & 20972.0 & 28888.0 & \textbf{46498.3}\\ 
StarGunner & 664.0 & 10250.0 & 54282.0 & 125117.0 & 77495.0 & 74677.0 & \textbf{131981.2}\\ 
Tennis & -23.8 & -9.3 & 12.2 & 0.0 & \textbf{23.6} & \textbf{23.6} & 22.6\\ 
TimePilot & 3568.0 & 5229.2 & 4870.0 & 7553.0 & 10345.0 & 12236.0 & \textbf{14995.2}\\ 
Tutankham & 11.4 & 167.6 & 68.1 & 245.9 & 297.0 & 293.0 & \textbf{309.2}\\ 
UpNDown & 533.4 & 11693.2 & 9989.9 & 33879.1 & 71260.0 & \textbf{88148.0} & 75474.4\\ 
Venture & 0.0 & 1187.5 & 163.0 & 48.0 & 43.9 & \textbf{1318.0} & 1112\\ 
VideoPinball & 16256.9 & 17667.9 & 196760.4 & 479197.0 & 705662.0 & 698045.0 & \textbf{799155.6}\\ 
WizardOfWor & 563.5 & 4756.5 & 2704.0 & 12352.0 & 25061.0 & 31190.0 & \textbf{44782.6}\\ 
YarsRevenge & 3092.9 & 54576.9 & 18098.9 & \textbf{69618.1} & 26447.0 & 28379.0 & 27691.2\\ 
Zaxxon & 32.5 & 9173.3 & 5363.0 & 13886.0 & 13113.0 & \textbf{21772.0} & 15179.5\\ 

	\end{tabular}
\caption{Raw scores for a single seed across all games, starting with 30 no-op actions.}  \label{bigtable}
\end{table*}

To align with previous works, the scores are evaluated under 30 no-op setting. As the sticky action evaluation setting proposed by~\cite{machado2018revisiting} is generally considered more meaningful in the RL community, we will add results under sticky-action evaluation setting after the conference.

%% file: main.bbl
\begin{thebibliography}{25}
\providecommand{\natexlab}[1]{#1}
\providecommand{\url}[1]{\texttt{#1}}
\expandafter\ifx\csname urlstyle\endcsname\relax
  \providecommand{\doi}[1]{doi: #1}\else
  \providecommand{\doi}{doi: \begingroup \urlstyle{rm}\Url}\fi

\bibitem[Barth-Maron et~al.(2018)Barth-Maron, Hoffman, Budden, Dabney, Horgan,
  Muldal, Heess, and Lillicrap]{barth2018distributed}
Gabriel Barth-Maron, Matthew~W Hoffman, David Budden, Will Dabney, Dan Horgan,
  Alistair Muldal, Nicolas Heess, and Timothy Lillicrap.
\newblock Distributed distributional deterministic policy gradients.
\newblock \emph{International Conference on Learning Representations}, 2018.

\bibitem[Bellemare et~al.(2013)Bellemare, Naddaf, Veness, and
  Bowling]{bellemare2013arcade}
Marc~G Bellemare, Yavar Naddaf, Joel Veness, and Michael Bowling.
\newblock The arcade learning environment: An evaluation platform for general
  agents.
\newblock \emph{Journal of Artificial Intelligence Research}, 47:\penalty0
  253--279, 2013.

\bibitem[Bellemare et~al.(2017)Bellemare, Dabney, and
  Munos]{bellemare2017distributional}
Marc~G Bellemare, Will Dabney, and R{\'e}mi Munos.
\newblock A distributional perspective on reinforcement learning.
\newblock In \emph{Proceedings of the 34th International Conference on Machine
  Learning-Volume 70}, pages 449--458. JMLR. org, 2017.

\bibitem[Bellman(1957)]{Bellman:1957}
Richard Bellman.
\newblock \emph{Dynamic Programming}.
\newblock Princeton University Press, Princeton, NJ, USA, 1 edition, 1957.

\bibitem[Castro et~al.(2018)Castro, Moitra, Gelada, Kumar, and
  Bellemare]{castro18dopamine}
Pablo~Samuel Castro, Subhodeep Moitra, Carles Gelada, Saurabh Kumar, and
  Marc~G. Bellemare.
\newblock Dopamine: {A} {R}esearch {F}ramework for {D}eep {R}einforcement
  {L}earning.
\newblock 2018.
\newblock URL \url{http://arxiv.org/abs/1812.06110}.

\bibitem[Dabney et~al.(2018{\natexlab{a}})Dabney, Ostrovski, Silver, and
  Munos]{dabney2018implicit}
Will Dabney, Georg Ostrovski, David Silver, and Remi Munos.
\newblock Implicit quantile networks for distributional reinforcement learning.
\newblock In \emph{International Conference on Machine Learning}, pages
  1104--1113, 2018{\natexlab{a}}.

\bibitem[Dabney et~al.(2018{\natexlab{b}})Dabney, Rowland, Bellemare, and
  Munos]{dabney2018distributional}
Will Dabney, Mark Rowland, Marc~G Bellemare, and R{\'e}mi Munos.
\newblock Distributional reinforcement learning with quantile regression.
\newblock In \emph{Thirty-Second AAAI Conference on Artificial Intelligence},
  2018{\natexlab{b}}.

\bibitem[Hessel et~al.(2018)Hessel, Modayil, Van~Hasselt, Schaul, Ostrovski,
  Dabney, Horgan, Piot, Azar, and Silver]{hessel2018rainbow}
Matteo Hessel, Joseph Modayil, Hado Van~Hasselt, Tom Schaul, Georg Ostrovski,
  Will Dabney, Dan Horgan, Bilal Piot, Mohammad Azar, and David Silver.
\newblock Rainbow: Combining improvements in deep reinforcement learning.
\newblock In \emph{Thirty-Second AAAI Conference on Artificial Intelligence},
  2018.

\bibitem[Huber(1964)]{huber:1964}
Peter~J. Huber.
\newblock Robust estimation of a location parameter.
\newblock \emph{Annals of Mathematical Statistics}, 35\penalty0 (1):\penalty0
  73--101, March 1964.
\newblock ISSN 0003-4851.
\newblock \doi{10.1214/aoms/1177703732}.

\bibitem[Jaquette et~al.(1973)]{jaquette1973markov}
Stratton~C Jaquette et~al.
\newblock Markov decision processes with a new optimality criterion: Discrete
  time.
\newblock \emph{The Annals of Statistics}, 1\penalty0 (3):\penalty0 496--505,
  1973.

\bibitem[Kapturowski et~al.(2018)Kapturowski, Ostrovski, Quan, Munos, and
  Dabney]{kapturowski2018recurrent}
Steven Kapturowski, Georg Ostrovski, John Quan, Remi Munos, and Will Dabney.
\newblock Recurrent experience replay in distributed reinforcement learning.
\newblock 2018.

\bibitem[Machado et~al.(2018)Machado, Bellemare, Talvitie, Veness, Hausknecht,
  and Bowling]{machado2018revisiting}
Marlos~C Machado, Marc~G Bellemare, Erik Talvitie, Joel Veness, Matthew
  Hausknecht, and Michael Bowling.
\newblock Revisiting the arcade learning environment: Evaluation protocols and
  open problems for general agents.
\newblock \emph{Journal of Artificial Intelligence Research}, 61:\penalty0
  523--562, 2018.

\bibitem[Mnih et~al.(2015)Mnih, Kavukcuoglu, Silver, Rusu, Veness, Bellemare,
  Graves, Riedmiller, Fidjeland, Ostrovski, et~al.]{mnih2015humanlevel}
Volodymyr Mnih, Koray Kavukcuoglu, David Silver, Andrei~A Rusu, Joel Veness,
  Marc~G Bellemare, Alex Graves, Martin Riedmiller, Andreas~K Fidjeland, Georg
  Ostrovski, et~al.
\newblock Human-level control through deep reinforcement learning.
\newblock \emph{Nature}, 518\penalty0 (7540):\penalty0 529, 2015.

\bibitem[Morimura et~al.(2010)Morimura, Sugiyama, Kashima, Hachiya, and
  Tanaka]{morimura2010nonparametric}
Tetsuro Morimura, Masashi Sugiyama, Hisashi Kashima, Hirotaka Hachiya, and
  Toshiyuki Tanaka.
\newblock Nonparametric return distribution approximation for reinforcement
  learning.
\newblock In \emph{Proceedings of the 27th International Conference on Machine
  Learning (ICML-10)}, pages 799--806, 2010.

\bibitem[M{\"u}ller(1997)]{muller1997integral}
Alfred M{\"u}ller.
\newblock Integral probability metrics and their generating classes of
  functions.
\newblock \emph{Advances in Applied Probability}, 29\penalty0 (2):\penalty0
  429--443, 1997.

\bibitem[Puterman(1994)]{Puterman:1994:MDP:528623}
Martin~L. Puterman.
\newblock \emph{Markov Decision Processes: Discrete Stochastic Dynamic
  Programming}.
\newblock John Wiley \& Sons, Inc., New York, NY, USA, 1st edition, 1994.
\newblock ISBN 0471619779.

\bibitem[Rowland et~al.(2018)Rowland, Bellemare, Dabney, Munos, and
  Teh]{rowland2018analysis}
Mark Rowland, Marc Bellemare, Will Dabney, Remi Munos, and Yee~Whye Teh.
\newblock An analysis of categorical distributional reinforcement learning.
\newblock In \emph{International Conference on Artificial Intelligence and
  Statistics}, pages 29--37, 2018.

\bibitem[Schaul et~al.(2016)Schaul, Quan, Antonoglou, and
  Silver]{schaul2016prioritized}
Tom Schaul, John Quan, Ioannis Antonoglou, and David Silver.
\newblock Prioritized experience replay.
\newblock \emph{International Conference on Learning Representations},
  abs/1511.05952, 2016.

\bibitem[Sobel(1982)]{sobel1982variance}
Matthew~J Sobel.
\newblock The variance of discounted markov decision processes.
\newblock \emph{Journal of Applied Probability}, 19\penalty0 (4):\penalty0
  794--802, 1982.

\bibitem[Sutton(1988)]{sutton1988learning}
Richard~S Sutton.
\newblock Learning to predict by the methods of temporal differences.
\newblock \emph{Machine learning}, 3\penalty0 (1):\penalty0 9--44, 1988.

\bibitem[Tang and Agrawal(2018)]{tang2018exploration}
Yunhao Tang and Shipra Agrawal.
\newblock Exploration by distributional reinforcement learning.
\newblock In \emph{Proceedings of the 27th International Joint Conference on
  Artificial Intelligence}, pages 2710--2716. AAAI Press, 2018.

\bibitem[Van~Hasselt et~al.(2016)Van~Hasselt, Guez, and Silver]{van2016deep}
Hado Van~Hasselt, Arthur Guez, and David Silver.
\newblock Deep reinforcement learning with double q-learning.
\newblock In \emph{Thirtieth AAAI Conference on Artificial Intelligence}, 2016.

\bibitem[Wang et~al.(2016)Wang, Schaul, Hessel, Hasselt, Lanctot, and
  Freitas]{wang2016dueling}
Ziyu Wang, Tom Schaul, Matteo Hessel, Hado Hasselt, Marc Lanctot, and Nando
  Freitas.
\newblock Dueling network architectures for deep reinforcement learning.
\newblock In \emph{International Conference on Machine Learning}, pages
  1995--2003, 2016.

\bibitem[Watkins(1989)]{watkins1989learning}
Christopher John Cornish~Hellaby Watkins.
\newblock Learning from delayed rewards.
\newblock 1989.

\bibitem[White(1988)]{white1988mean}
DJ~White.
\newblock Mean, variance, and probabilistic criteria in finite markov decision
  processes: a review.
\newblock \emph{Journal of Optimization Theory and Applications}, 56\penalty0
  (1):\penalty0 1--29, 1988.

\end{thebibliography}
